\documentclass{article}

\usepackage[preprint, nonatbib]{neurips_2026}  

\usepackage[utf8]{inputenc} 
\usepackage[T1]{fontenc}    
\usepackage{hyperref}       
\usepackage{url}            
\usepackage{booktabs}       
\usepackage{amsfonts}       
\usepackage{nicefrac}       
\usepackage{microtype}      
\usepackage{xcolor}         

\usepackage{graphicx}       
\usepackage{amsmath}        
\usepackage{amssymb}        
\usepackage{csquotes}       
\usepackage{pdfpages}       
\usepackage{float}
\usepackage{multirow}
\usepackage[table,dvipsnames]{xcolor}
\usepackage{fontawesome5}
\usepackage{todonotes}
\usepackage{mathtools}

\hypersetup{
    colorlinks=true,
    linkcolor=MidnightBlue,
    urlcolor=MidnightBlue,
}

\def\Pid{P_{\text{ID}}}
\def\Pood{P_{\text{OOD}}}
\def\Pnoise{P_{0}}
\def\X{X}
\def\Z{Z}

\usepackage{amsthm}
\newtheorem{theorem}{Theorem}[section]
\newtheorem{lemma}[theorem]{Lemma}
\newtheorem{corollary}[theorem]{Corollary}

\usepackage{cleveref}       
\crefname{figure}{Figure}{Figures}
\crefname{table}{Table}{Tables}
\crefname{section}{Section}{Sections}
\crefname{appendix}{Appendix}{Appendices}

\newcommand{\cibracket}[2]{\,{\scriptsize [#1, #2]}}

\makeatletter
\newcommand{\reviewhide}[1]{%
  \if@neuripsfinal
    #1
  \else
    \if@preprint
      #1
    \else
      [removed for review]
    \fi
  \fi
}
\makeatother

\usepackage[backend=biber, style=numeric, sorting=none]{biblatex}
\AtEveryBibitem{%
  \clearfield{note}%
  \clearlist{language}
}
\renewbibmacro*{in:}{%
  \iffieldundef{journaltitle}{}{%
    \printtext{\bibstring{in}\intitlepunct}}%
  \iffieldundef{booktitle}{}{%
    \printtext{\bibstring{in}\intitlepunct}}%
}
\renewbibmacro*{url+urldate}{}

\title{The Signal in the Noise: OOD Detection Through Goodness-of-Fit Testing in Factorised Latent Spaces}

\author{%
  Philipp Bomatter \\
  School of Informatics \\
  University of Edinburgh \\
  \texttt{philipp.bomatter@ed.ac.uk} \\
  \And
  Jack Geary \\
  School of Informatics \\
  University of Edinburgh \\
  \texttt{jack.geary@ed.ac.uk} \\
  \And
  Henry Gouk \\
  School of Informatics \\
  University of Edinburgh \\
  \texttt{henry.gouk@ed.ac.uk} \\
}

\begin{document}

\maketitle

\begin{abstract}
    Deep generative models offer a natural foundation for out-of-distribution (OOD) detection, yet prior work has shown that their assigned likelihoods are notoriously unreliable indicators for in- vs out-of-distribution data. In this paper, we address this problem by leveraging the diffeomorphic and mass-preserving properties of continuous normalising flows. Our analysis shows that OOD samples are mapped to noise samples that are highly atypical under the noise prior in ways not captured by the likelihood. Based on this observation, we propose a new method---Signal in the Noise (SITN)---for OOD detection on the single-sample level. SITN requires no access to OOD data, incurs minimal computational overhead, and provides strict control of false positive rates. Comprehensive evaluations through standard benchmarks and synthetic perturbations highlight the method's effectiveness and the absence of the complexity bias inherent to likelihood-based methods.
\end{abstract}

\begin{center}
    \faGithub \quad \reviewhide{\href{https://github.com/bomatter/signal-in-the-noise-paper}{github.com/bomatter/signal-in-the-noise-paper}}
\end{center}

\section{Introduction}
\label{sec:introduction}

    The deployment of machine learning models in high-risk domains like healthcare is fundamentally limited by their reliability. Models are prone to failing catastrophically---and with high confidence---when confronted with out-of-distribution (OOD) data that deviates substantially from the training distribution \cite{nguyen_deep_2014}. To mitigate this risk, OOD detection can enable systems to abstain from decision-making or defer to a human expert when a model's prediction is deemed unreliable.

    Deep generative models provide a natural foundation for OOD detection by explicitly modelling the data distribution. In particular, Continuous Normalising Flows (CNFs) \cite{chen2018neural}, often trained via flow matching methods \cite{lipman_flow_2023}, are a logical choice for this task because they enable the tractable computation of exact likelihoods. However, prior work has shown that these likelihoods are unreliable indicators for in- vs out-of-distribution data in high dimensions \cite{nalisnick_deep_2019}. For instance, widely replicated experiments have demonstrated that generative models trained on CIFAR-10 consistently assign higher likelihoods to samples from MNIST or SVHN \cite{nalisnick_deep_2019, serra_input_2020, schirrmeister_understanding_2020, kirichenko_why_2020}.

    Previous work has investigated the theoretical underpinnings of this limitation of likelihood-based OOD detection and proposed various corrections to mitigate it. Notable approaches include typicality tests (essentially a two-sided thresholding of the likelihood) \cite{nalisnick_detecting_2019}, complexity penalties \cite{serra_input_2020}, likelihood ratios \cite{schirrmeister_understanding_2020, ren_likelihood_2019, xiao_likelihood_2020}, and fine-grained analyses of likelihood components \cite{nalisnick_deep_2019, morningstar_density_2021}. While these methods succeed in partially alleviating the pronounced failure mode observed between datasets like CIFAR-10 and SVHN, our analysis of representative approaches (Typicality \cite{nalisnick_detecting_2019}, DoSE \cite{morningstar_density_2021}) reveals that they still exhibit undesirable biases linked to inherent limitations of the likelihood. Furthermore, these approaches frequently impose practical constraints, as they may require domain-specific compression algorithms, depend on (typically unavailable) out-of-distribution data to construct background models, or incur significant computational overhead.

    We depart from these likelihood-centric approaches and propose a new framework. Leveraging the mass-conserving, diffeomorphic mappings of CNFs, we first show that OOD samples map to noise latents with atypical empirical properties under the noise prior. Notably, this holds independently of the divergence of the vector field, which does affect likelihoods. We then quantify this atypicality at the single-sample level by exploiting the completely factorised nature of the Gaussian noise prior. Specifically, this factorisation implies that the elements of a single high-dimensional latent vector must behave as i.i.d. samples from a univariate standard normal distribution. This allows us to test an individual sample based on two defining properties: the marginal normality of its elements and their dimensional independence. We present a new method termed Signal in the Noise (SITN) as a specific instantiation of this framework, and validate it through comprehensive evaluations on standard benchmarks and synthetic perturbations.
    
    In summary, our main contributions are as follows:
    \begin{itemize}
        \item We introduce a new, theoretically grounded framework that casts OOD detection as single-sample goodness-of-fit testing against the factorised Gaussian noise prior of Continuous Normalising Flows.

        \item We propose a specific method (SITN) for OOD detection. SITN incurs minimal computational overhead, requires no access to OOD data, and provides strict Type I error control. We demonstrate the method's effectiveness through standard benchmarks and detailed ablations.

        \item We present visualisations and analyses that shed light on the differential behaviour of SITN and various baseline methods and highlight promising directions for future work.
    \end{itemize}

\section{Related Work}
\label{sec:related_work}
    
    \paragraph{Limitations of the Likelihood for OOD Detection}
        Several works have identified that the likelihood assigned by deep generative models is not a reliable indicator for OOD detection \cite{nalisnick_deep_2019, serra_input_2020, schirrmeister_understanding_2020, kirichenko_why_2020, havtorn_hierarchical_2022, kamkari_geometric_2024}. \citeauthor*{nalisnick_detecting_2019} highlighted an inherent limitation of the likelihood: regions of high probability \emph{density} can lie outside the typical set of the distribution that contains almost all probability \emph{mass} \cite{nalisnick_detecting_2019}. This is well illustrated by the concentration of measure behaviour of a high-dimensional ($d \gg 1$) standard Gaussian distribution, where the maximum density is attained at the mode, but virtually all probability mass is concentrated in the annulus with radius $\sqrt{d}$ \cite{blum_foundations_2020}. \citeauthor*{zhang2021understanding} challenged the typical set hypothesis as an explanation for the observed failures, attributing them instead to model estimation error \cite{zhang2021understanding}. Other works have argued that likelihoods obtained through generative models are ineffective for OOD detection because they are dominated by low-level features that are similar across datasets, rather than semantic content \cite{schirrmeister_understanding_2020, kirichenko_why_2020, havtorn_hierarchical_2022}.

    \paragraph{Likelihood Corrections}
        To address these limitations, several works have attempted to correct or augment the raw likelihoods.
        Based on their concentration of measure observations, 
        \citeauthor*{nalisnick_detecting_2019} introduced typicality tests, arguing that in-distribution (ID) samples should not simply lie in a high density region, but rather reside in the typical set \cite{nalisnick_detecting_2019}.
        Extending this geometric intuition, \citeauthor*{kamkari_geometric_2024} proposed pairing likelihoods with estimates of Local Intrinsic Dimension (LID) to flag OOD samples that reside on low-dimensional manifolds---regions that contain negligible probability mass despite potentially high likelihoods \cite{kamkari_geometric_2024}. However, estimating LID for normalising flows relies on explicitly constructing the full Jacobian matrix. While this is tractable for heavily constrained architectures (e.g., Glow), it is computationally prohibitive for modern, highly expressive flow matching models.
        \citeauthor*{serra_input_2020} observed a strong correlation between likelihood and sample complexity (e.g., PNG compression size) and proposed a complexity correction that can be interpreted as a likelihood ratio test \cite{serra_input_2020}. This requires a specific complexity estimate, which can be arbitrary and severely limiting in non-standard data domains where suitable compression algorithms are unavailable.
        Other works have also proposed likelihood ratio tests \cite{ren_likelihood_2019, schirrmeister_understanding_2020}. A downside of these approaches is the requirement of a separate background model trained on OOD data, which is not always straightforward to construct and adds computational overhead.
        A special case that bypasses the need for OOD data is the Likelihood Regret method proposed by \citeauthor*{xiao_likelihood_2020} \cite{xiao_likelihood_2020}. Their method fine-tunes the generative model on the individual test sample (through additional optimisation steps during inference) and then computes the log ratio between the likelihoods obtained from the fine-tuned and the baseline model. As a downside, this approach complicates the inference process and still incurs computational overhead.
        Alternatively, \citeauthor*{choi_waic_2019} proposed using the Watanabe-Akaike Information Criterion (WAIC) to penalize epistemic uncertainty by calculating the variance of likelihood evaluations across an ensemble of generative models, albeit at the cost of training multiple networks \cite{choi_waic_2019}.

    \paragraph{Beyond Likelihoods}
        More recently, methods have moved beyond purely data-space likelihoods to probe the latent representations and generative mappings of the generative models.
        Most similar in spirit to our method is the work by \citeauthor*{jiang_revisiting_2021} \cite{jiang_revisiting_2021}, which also explored goodness-of-fit testing in the latent space. Their work is however focused on group-OOD detection and not directly applicable to single samples as our method.
        Density of States Estimation (DoSE) measures the empirical density of multiple statistics, including the log-probability of the latent representation and the log-determinant of the Jacobian of the mapping \cite{morningstar_density_2021}.
        \citeauthor*{heng_out--distribution_2024} proposed DiffPath, which measures the rate-of-change and curvature of diffusion paths \cite{heng_out--distribution_2024}. Interestingly, their approach relies on the ability of diffusion models to map even OOD samples approximately to the noise prior. Our method leverages the exact opposite: using deterministic and diffeomorphic CNFs instead of stochastic diffusion models, we exploit that OOD and ID samples are not mapped to the same regions in the noise space.
    
\section{Methods}
\label{sec:methods}

    \subsection{Out-of-distribution Detection in Noise Space}
    
        We formalise the problem of OOD detection as classifying a novel observation, $\X$, as being drawn from one of two possible distributions: $\Pid$ (In Distribution) or $\Pood$ (Out Of Distribution). During training time, it is assumed one only has access to data drawn from $\Pid$.
        
        We begin by considering a CNF trained via, e.g., Flow Matching \cite{lipman_flow_2023}. A perfectly trained CNF defines a vector field whose corresponding Ordinary Differential Equation (ODE) admits a flow  $\phi : \mathbb{R}^D \to \mathbb{R}^D$ that maps samples from a prior $\Pnoise$ to samples from $\Pid$; i.e., the push-forward measure of $\Pnoise$ by $\phi$ is $\Pid$. $\Pnoise$ is usually taken to be $\mathcal{N}(0, I)$---we will assume this choice going forward. The reversibility of ODEs means one also has access to $\phi^{-1}$, and is able to map real data points back to samples from $\Pnoise$. Consequently, mapping ID data points back to the noise space will result in samples from $\mathcal{N}(0, I)$. Conversely, OOD samples from $\Pood$ will yield samples that are not distributed according to the standard multivariate Gaussian. That is, for $\Z = \phi^{-1}(\X)$ we have that $\Z \sim \Pnoise$ if and only if $\X \sim \Pid$. We quantify the extent to which $\Z$ conforms with the prior through test statistics, reducing OOD detection to a statistical goodness-of-fit test in the noise space.
    
        \begin{figure}[t]
            \centering
            \includegraphics[width=\textwidth]{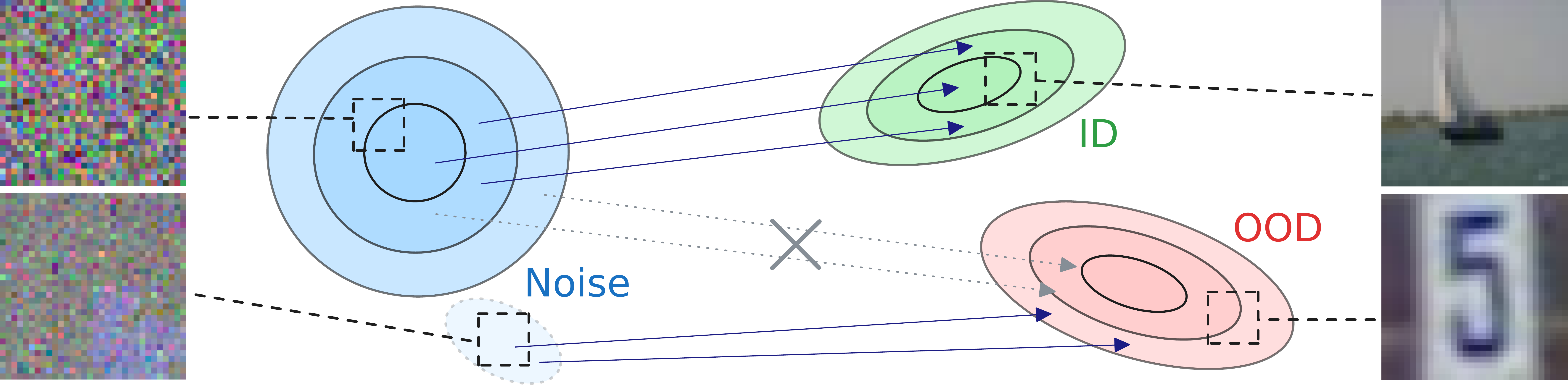}
            \caption{
                Illustration of OOD detection in the noise space. CNF models are trained to map noise samples from a standard normal distribution to the target distribution (ID). Since the resulting flow is a probability mass-conserving diffeomorphism, OOD samples with low probability mass under the target distribution will not correspond to typical Gaussian noise samples in noise space (crossed dotted lines), such that OOD detection can be cast as a goodness of fit test in the noise space.
            }
            \label{fig:concept}
        \end{figure}

    \subsection{Goodness-of-Fit Testing Against the Prior}

        To determine the goodness-of-fit of some $\Z$ against the theoretical standard multivariate Gaussian, we probe its two defining properties: marginal normality (each element $\Z_i$ follows a standard univariate normal distribution) and dimensional independence (elements $\Z_i$ and $\Z_j$ are mutually independent for all $i \neq j$). Formally, we treat this as a goodness-of-fit test with a composite null hypothesis, $H_0$, and decompose it into a conjunction of two constituent hypotheses $H_{\text{MN}}$ and $H_{\text{DI}}$, for marginal normality and dimensional independence, respectively: $H_0 \Rightarrow H_{\text{MN}} \land H_{\text{DI}}$. Under this decomposition, each constituent hypothesis poses a necessary condition for $H_0$. That is, if we reject one of the constituent hypotheses, then we should also reject the null hypothesis.
        
        Crucially, we are able to test these hypotheses entirely at the single-sample level by exploiting the fact that the multivariate standard Gaussian prior factorises completely. Specifically, a single observation, $\Z \sim \mathcal{N}(0, I)$, can be equivalently viewed as a collection of $D$ i.i.d. samples from a univariate standard Gaussian, $\mathcal{N}(0, 1)$. Because the dimensionality $D$ of many standard data modalities (such as images) is large, this perspective allows us to compute robust empirical statistics from an individual high-dimensional sample, rather than requiring a batch of multiple samples.

        In the following, we detail the specific statistics and combination mechanism that form our proposed SITN method as an instantiation of this framework.

        \paragraph{Marginal Normality}
            To evaluate whether the individual elements of $\Z$ conform to the expected univariate normal distribution, we employ the Anderson-Darling (AD) statistic. By treating the $D$ dimensions of $\Z$ as an empirical sample, the AD statistic quantifies the discrepancy between the empirical cumulative distribution function (CDF) of the sample and the theoretical CDF of the standard normal distribution, $\Phi$. Sorting the elements of $\Z$ in ascending order such that $\Z_{(1)} \leq \Z_{(2)} \leq \dots \leq \Z_{(D)}$, the statistic is computed as
            \begin{equation}
                S_{\text{AD}} = -D - \sum_{i=1}^{D} \frac{2i - 1}{D} \left[ \ln \Phi(\Z_{(i)}) + \ln(1 - \Phi(\Z_{(D+1-i)})) \right].
            \end{equation}
            A high $S_{\text{AD}}$ indicates that $\phi^{-1}$ mapped the data sample to a noise sample with an anomalous mean, incorrect variance, or distorted distributional shape violating the marginal normality expectation of the prior.

        \paragraph{Dimensional Independence}
            Acknowledging that no single statistical test can comprehensively detect all forms of dependence \cite{hutter_testing_2022}, we specifically test for autocorrelation. Independence between random variables implies there is no autocorrelation, so this still represents a necessary condition for the null hypothesis.
            We compute the coefficient of variation (CV) of the empirical power spectrum (PSD) of $\Z$, which is defined as the ratio of the standard deviation of the power spectrum to its arithmetic mean,
            \begin{equation}
                S_{\text{CV}} = \frac{\sigma_{\text{PSD}}}{\mu_{\text{PSD}}} = \frac{\sqrt{\frac{1}{D} \sum_{k=1}^{D} (\text{PSD}(k) - \mu_{\text{PSD}})^2}}{\mu_{\text{PSD}}},
            \end{equation}
            where $\mu_{\text{PSD}} = \frac{1}{D} \sum_{k=1}^{D} \text{PSD}(k)$. For standard Gaussian white noise, the empirical power spectrum follows an exponential distribution ($\sigma = \mu$), yielding a theoretical $S_{\text{CV}}$ of exactly $1$. Conversely, structural correlations between the components of $\Z$ manifest as dominant frequencies, inflating the spectral variance and driving $S_{\text{CV}}$ significantly above $1$. This statistic serves as a global test for autocorrelation: by the Wiener--Khinchin theorem, the power spectrum and autocorrelation function form a Fourier transform pair, meaning that an inflated $S_{\text{CV}}$ mathematically implies non-zero autocorrelation. We prefer this spectral approach over time-domain alternatives like the Ljung-Box test, as it captures global dependencies without the need to select a max-lag hyperparameter.

    \paragraph{Combining Statistics}
    \label{sec:statistics_combination}

        To produce a single scalar OOD score, we combine the normality and independence statistics using a max-quantile strategy. The empirical cumulative distribution functions (CDFs) for both statistics, denoted as $\hat{F}_{\text{AD}}$ and $\hat{F}_{\text{CV}}$, are estimated using an ID calibration set.
        The combined score is then defined as the maximum of the quantile-transformed statistics
        \begin{equation}
            S_{\text{SITN}} = \max \left( \hat{F}_{\text{AD}}(S_{\text{AD}}), \hat{F}_{\text{CV}}(S_{\text{CV}}) \right).
        \end{equation}
        The use of quantile transforms ensures that the statistics are mapped into a space where they have the same range and marginal distribution (i.e., uniform on the $\lbrack 0, 1 \rbrack$ interval), thus allowing them to be meaningfully compared. By taking the maximum, we identify a sample as OOD if it appears as an outlier in either its marginal distribution or its structural independence, achieving our goal of rejecting $H_0$ if any of the constituent hypotheses should be rejected.

    \subsection{Theoretical Properties}
    
        \paragraph{Provable Control of False Positive Rate}
            Determining a suitable threshold for OOD scores is difficult in practice. Ideally, one would have access to a set of known OOD and ID data, and use this to choose a threshold that trades off between the false positive and false negative rates of the OOD detector. While this is not possible without access to OOD data, ID data is sufficient to calibrate the detector, allowing one to select a tolerable false positive rate. In \cref{appendix:calibration}, we describe how $S_{\text{SITN}}$ can be calibrated and prove that this provides false positive rate (i.e., Type I error) control.
    
        \paragraph{Robustness to CNF Approximation Error}
            CNFs typically do not learn an exact model of the target $\Pid$, but rather a reasonably good approximation. Our framework is still applicable when the CNF is not a perfect representation of $\Pid$, but one must instead interpret the constituent hypotheses, $H_{\text{MN}}$ and $H_{\text{DI}}$, differently. They should be understood as statements about the maximum acceptable levels of deviation from normality and independence, respectively. One should expect poorer quality CNFs to result in increased false negative (i.e., Type II error) rates, but we note that the false positive rate is not impacted.
    
        \paragraph{Computational Complexity}
            For a noise sample $\Z$ of dimensionality $D$, the computational cost of $S_{\text{AD}}$ is dominated by the sorting of the elements, resulting in a complexity of $\mathcal{O}(D \log D)$. Similarly, the computation of $S_{\text{CV}}$ relies on obtaining the empirical power spectrum, which is dominated by the Fast Fourier Transform (FFT) and also scales as $\mathcal{O}(D \log D)$. The combined overhead of computing these statistics is negligible compared to the computational cost of the neural vector field evaluations required for the inverse flow.

    \section{Experimental Design}
    \label{sec:experimental_design}

        For each dataset described in \cref{sec:datasets}, we train a single unconditional CNF and use it to evaluate the OOD detection performance of the baseline methods described in \cref{sec:baselines} and SITN. For the cross-dataset experiments, we use the test split of the training dataset as ID data and the test split of a different dataset as OOD data. For the CIFAR-10 model, we furthermore evaluate OOD detection performance using the synthetic perturbations from the CIFAR-10-C benchmark \cite{hendrycks_benchmarking_2019}.

        Importantly, for comparisons across datasets with different original resolutions (e.g. CIFAR-10 and CelebA), we downsample the higher resolution dataset. This is to prevent introducing blurriness that leads to trivial OOD detection as reported by \citeauthor*{heng_out--distribution_2024} \cite{heng_out--distribution_2024}. In their experiments, the OOD detection AUROC on CIFAR-10 vs CelebA of their model dropped from 0.999 (when the lower resolution dataset was upsampled) to 0.502 (when the higher resolution dataset was downsampled).

        We use the UNet architecture following \citeauthor*{ho_denoising_2020} \cite{ho_denoising_2020}, as implemented in \citeauthor*{dhariwal_diffusion_2021} \cite{dhariwal_diffusion_2021}. Following the flow matching objective of \citeauthor*{lipman_flow_2023} \cite{lipman_flow_2023}, we construct a linear interpolant between a Gaussian noise sample $\Z \sim \mathcal{N}(0,I)$ and a data sample $\X \sim \Pid$: $\X_t = (1-t)\Z + t\X, \quad t \sim \mathcal{U}[0,1]$, and minimise the mean-squared error between the network output and the target vector field $\X - \Z$. Full architectural and training hyperparameters are detailed in Appendix~\ref{appendix:training_details}.
        
        The adaptive Dormand-Prince ODE solver (dopri5) \cite{dormand_family_1980} is used to compute $\phi$ and $\phi^{-1}$. OOD detection scores are then derived using the baseline methods and SITN.
    
        \subsection{Datasets}
        \label{sec:datasets}
        
        \paragraph{CIFAR-10}
            CIFAR-10 \cite{krizhevsky_learning_2009} consists of 60{,}000 $32 \times 32$ images
            across 10 object categories (50{,}000 train, 10{,}000 test). We reserve 10\% of the training set as validation split, yielding 45{,}000 training, 5{,}000 validation, and 10{,}000 test images.
        
        \paragraph{SVHN}
            The Street View House Numbers (SVHN) dataset \cite{netzer_reading_2011} contains $32 \times 32$ images obtained from Google Street View and is available for non-commercial use. We use the standard train (73{,}257 images) and test (26{,}032 images) splits, reserving 10\% of the training set for validation.
        
        \paragraph{CelebA}
            CelebA \cite{liu_deep_2015} comprises over 200{,}000 celebrity images with attribute annotations. It is available for non-commercial research purposes. We adopt the predefined train, validation, and test splits, downsampling all images to $32 \times 32$ pixels for consistency with the other datasets.

        \subsection{Baseline Methods}
        \label{sec:baselines}
            
            \paragraph{Log-likelihood}
                The log-likelihood is evaluated through the instantaneous change of variables formula as the sum of the log-probability of the latent representation and the integral of the vector field's divergence.

            \paragraph{Typicality}
                Typicality was proposed by \citeauthor*{nalisnick_detecting_2019} \cite{nalisnick_detecting_2019} and is defined as the absolute difference between the sample's negative log-likelihood and the resubstitution estimator of the entropy (i.e., the average negative log-likelihood over the training data; we experiment with using validation data instead in \cref{appendix:baselines_on_val}).
                We use the single-sample version as in \cite{morningstar_density_2021}.
            
            \paragraph{DoSE}
                Density of States Estimation (DoSE) \cite{morningstar_density_2021} measures the empirical density of multiple statistics using the (in-distribution) training data (we experiment with using validation data instead in \cref{appendix:baselines_on_val}). Specifically, we implement the $\text{DoSE}_{\text{KDE}}$ variant developed for flow-based models, which uses three statistics: the log-likelihood, the log-probability of the latent representation, and the log-determinant of the Jacobian. Because we operate in the continuous-time regime of CNFs, the  log-determinant term is replaced with the integral of the vector field's divergence.
                A one-dimensional Kernel Density Estimator (KDE) is fit to each statistic using the training data and the final OOD score is computed by summing the log-probabilities from the individual KDEs. Following the original publication, we use the default SciPy KDE implementation with a Gaussian kernel and automatic determination of the bandwidth parameter through Scott's rule. For computational efficiency, the KDEs are fit on a random subset of 10,000 samples; we empirically verified that this subsampling has minimal impact on the results. 

\section{Results}
\label{sec:results}

    \subsection{OOD Detection Performance: Cross-dataset Experiments}
    \label{sec:cross_dataset_results}

        \cref{tab:cross_dataset_results} reports the Area Under the Receiver Operating Characteristic (AUROC) curve for the cross-dataset experiments. Models were trained on one dataset and then evaluated for their OOD detection performance by using the test split of the training dataset as in-distribution and the test split of a different dataset as out-of-distribution samples. Further results with a comparison against complexity correction and WAIC can be found in \cref{appendix:additional_baselines}.

        \begin{table}[h]
          \caption{Cross-dataset OOD detection performance in terms of AUROC. Values in brackets denote bootstrapped 95\% CIs and results for the best method are highlighted in \textbf{bold}. SITN consistently outperforms baselines, demonstrating that noise samples contain useful information for OOD detection.}
          \label{tab:cross_dataset_results}
          \centering
          \begin{tabular}{llcccc}
            \toprule
            \textbf{Train (ID)} & \textbf{Test (OOD)} & \textbf{Log-likelihood} & \textbf{Typicality} & \textbf{DoSE} & \textbf{SITN (Ours)} \\
            \midrule
            CIFAR-10 & SVHN     & 0.105\cibracket{.100}{.110} & 0.790\cibracket{.785}{.796} & 0.802\cibracket{.796}{.808} & \textbf{0.946}\cibracket{.943}{.949} \\
                     & CelebA   & 0.496\cibracket{.488}{.503} & 0.421\cibracket{.414}{.428} & 0.432\cibracket{.425}{.439} & \textbf{0.671}\cibracket{.665}{.678} \\
            \midrule
            SVHN     & CIFAR-10 & 0.980\cibracket{.978}{.981} & 0.954\cibracket{.951}{.957} & 0.984\cibracket{.982}{.985} & \textbf{0.986}\cibracket{.984}{.987} \\
                     & CelebA   & \textbf{0.999}\cibracket{.999}{.999} & 0.995\cibracket{.994}{.995} & 0.997\cibracket{.997}{.998} & 0.996\cibracket{.995}{.996} \\
            \midrule
            CelebA   & CIFAR-10 & 0.728\cibracket{.721}{.734} & 0.688\cibracket{.681}{.694} & 0.717\cibracket{.710}{.723} & \textbf{0.910}\cibracket{.907}{.914} \\
                     & SVHN     & 0.179\cibracket{.175}{.183} & 0.656\cibracket{.651}{.661} & 0.776\cibracket{.771}{.780} & \textbf{0.986}\cibracket{.985}{.987} \\
            \bottomrule
          \end{tabular}
        \end{table}

        In line with prior work, the log-likelihood results in high AUROC values when trained on the simpler SVHN dataset and evaluated against CIFAR-10 or CelebA, but fails consistently when the setup is reversed. Typicality improves on the raw log-likelihoods in setups where their distributions across the two datasets are well separated, with consistently higher likelihoods across the OOD dataset (e.g., CIFAR-10$\rightarrow$SVHN). However, it can fail when the likelihood distributions are strongly overlapping, particularly if the OOD dataset's likelihood distribution is more concentrated (e.g., CIFAR-10$\rightarrow$CelebA; see \cref{appendix:likelihood_distributions} for likelihood distribution plots). DoSE shows similar behaviour to Typicality, albeit with a small improvement. Our method, SITN, consistently outperforms these baselines, even in conditions where the baselines exhibit relatively low performance.
        
        Overall, these results provide empirical evidence that noise samples contain useful information for OOD detection and position SITN as an effective method to leverage that information.

    \subsection{OOD Detection Performance: Synthetic Perturbations}
    \label{sec:perturbation_results}

        \cref{fig:results_perturbations} shows OOD detection performance for the CIFAR-10-trained model across the synthetic perturbations from CIFAR-10-C. SITN shows the most consistent result with a macro average across conditions of 0.85 compared to 0.59 for the Log-likelihood, 0.69 for Typicality, and 0.78 for DoSE.

        SITN achieves its lowest performance for the contrast and fog corruptions. Both Typicality and DoSE perform significantly better on the former, but even worse on the latter. Conversely, SITN performs particularly well for JPEG, snow, or spatter where other methods struggle.

        Overall, these results provide further evidence for the effectiveness of goodness-of-fit testing against the factorised noise prior, showing that it is sensitive to a broad range of perturbations.

        \begin{figure}[t]
            \centering
            \includegraphics[width=\textwidth]{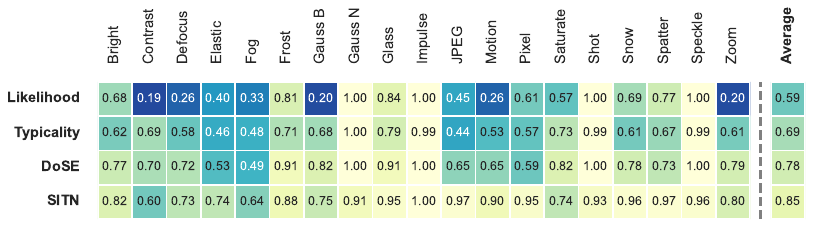}
            \caption{
                OOD detection performance across the CIFAR-10-C corruptions in terms of AUROC. For each corruption, samples with all severity levels were pooled. SITN performs best overall with the highest average performance across corruptions and a minimum AUROC of 0.6 on the contrast perturbation, whereas other methods drop to (or below) chance level in multiple conditions.
            }
            \label{fig:results_perturbations}
        \end{figure}

    \subsection{Metric Extremes}
    \label{sec:metric_extremes}

        In \cref{fig:metric_extremes_svhn} we show the images with highest and lowest OOD scores for each method in the CIFAR-10$\rightarrow$SVHN setup (trained on CIFAR-10, evaluated on the CIFAR-10 test split for ID and on the SVHN test split for OOD). For SITN, 18 samples were tied with a maximum OOD score of 1, from which 5 were selected, approximately preserving the ID/OOD ratio reflected in the 18-way tie. \cref{appendix:metric_extremes_celeba} shows the same visualisations for CIFAR-10$\rightarrow$CelebA, and \cref{appendix:top_150_ood} shows the 150 images with the highest OOD score for each method.

        The complexity-bias of the likelihood reported by \citeauthor*{serra_input_2020}\cite{serra_input_2020} is very conspicuous, as the highest likelihood scores are assigned to simple images with constant backgrounds. Typicality flags these samples as the most OOD, although they are technically considered in-distribution in this setting. This points to an inherent limitation of Typicality for datasets with a broad distribution of likelihoods like CIFAR-10, which contains both the highest and lowest likelihood samples in this case. DoSE inherits some of the same failure cases and additionally flags some of the highest likelihood samples as OOD. This makes sense given that one of the components considered in DoSE is a KDE of the likelihood, such that it will be sensitive to extremely high and low likelihoods. These results show that while Typicality and DoSE mitigate some of the catastrophic failures of using the raw likelihood for OOD detection, they are still affected by its inherent complexity bias. We illustrate this point further in \cref{appendix:complexity_bias} and show that SITN does not exhibit this bias.

        Overall, SITN is robust to the failure modes of likelihood-based methods and the only method to feature actual OOD images among the most OOD-scored samples. We discuss its failure cases in \cref{sec:noise_visualisations}.

        \begin{figure}[t]
            \centering
            \includegraphics[width=\textwidth]{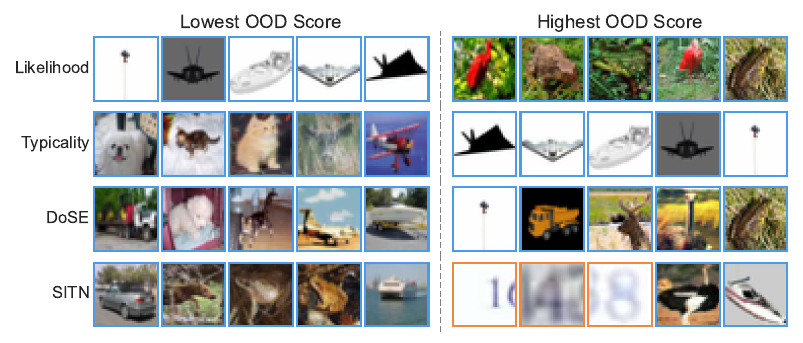}
            \caption{
                Images with the highest and lowest OOD scores for each method in the CIFAR-10$\rightarrow$SVHN setup, i.e. trained on CIFAR-10 and evaluated on the test split of CIFAR-10 as ID and the test split of SVHN as OOD. ID images from CIFAR-10 are marked with a \textcolor[RGB]{76, 155, 232}{blue} border and OOD images from SVHN with an \textcolor[RGB]{244, 136, 58}{orange} border.
                All baseline methods falsely flag in-distribution samples as most OOD, and both Typicality and DoSE inherit failure cases from the complexity-biased likelihood scores.
            }
            \label{fig:metric_extremes_svhn}
        \end{figure}

    \subsection{Ablations}
    \label{sec:ablations}
        
        \cref{appendix:perturbation_ablations} presents an ablation study, showing the performance of the individual metric components in SITN (AD for normality testing and CV for independence testing) across the corruptions from CIFAR-10-C. The two statistics show complementary value across the different perturbations. For instance, normality testing was more sensitive to brightness perturbations, whereas motion artefacts were picked up much better through independence testing.
        
        In \cref{tab:metric_aggregation_ablation}, we further show an ablation for the metric aggregation approach. We compare the max-quantile approach described in \cref{sec:statistics_combination} with the KDE approach used in DoSE. Furthermore, as both Typicality and DoSE use the training data to fit the entropy estimate and KDE, respectively, we also evaluate SITN with the max-quantile or KDE approach fitted on the training data.

        The max-quantile strategy performs better than the KDE approach, while the difference between fitting it on training vs validation data is negligible. If available, we prefer to use validation data to preclude overfitting issues and enable proper calibration as detailed in \cref{appendix:calibration}. We note that our max-quantile strategy cannot be trivially applied to combine the different metrics in DoSE, as it requires all metrics to have an inherent OOD direction.

        \begin{table}[h]
          \caption{Ablation study of SITN metric aggregation. The table shows AUROC performance with 95\% bootstrapped confidence intervals for the CIFAR-10-trained model. Max-quantile consistently outperforms the KDE approach, while the difference between fitting it on training vs validation data is negligible. The first row shows the default SITN configuration.}
          \label{tab:metric_aggregation_ablation}
          \centering
          \begin{tabular}{llcc}
            \toprule
            \textbf{Metric aggregation} & \textbf{Fit split} & \textbf{SVHN} & \textbf{CelebA} \\
            \midrule
            \rowcolor{gray!15} 
            Max Quantile       & Val       & 0.946 [0.943, 0.949] & 0.671 [0.665, 0.678] \\
            Max Quantile       & Train     & 0.948 [0.946, 0.951] & 0.672 [0.666, 0.679] \\
            KDE                & Val       & 0.915 [0.911, 0.918] & 0.615 [0.608, 0.621] \\
            KDE                & Train     & 0.913 [0.909, 0.917] & 0.617 [0.610, 0.624] \\
            \bottomrule
          \end{tabular}
        \end{table}
    
    \subsection{Noise Visualisations}
    \label{sec:noise_visualisations}
        \cref{fig:noise_samples_svhn} visualises images from the CIFAR-10$\rightarrow$SVHN setup (trained on CIFAR-10, evaluated on the CIFAR-10 test split as ID and the SVHN test split as OOD), alongside their corresponding noise samples obtained via backwards integration along the probability flow ODE. The left panel displays samples with the lowest SITN OOD scores—i.e., samples where both normality (assessed via AD) and independence (assessed via CV) are consistent with the standard Gaussian noise prior. The middle and right panels illustrate samples with high SITN OOD scores driven by violations of normality and independence, respectively.

        As expected, ID samples with low OOD scores map seamlessly to typical Gaussian noise. Furthermore, OOD samples flagged by the CV component of SITN exhibit strong checkerboard-like patterns, clearly violating dimensional independence. Interestingly, samples flagged by the AD component contain large patches of constant background, a structure that is visibly retained in the noise space. Because this phenomenon occurs not only in OOD samples but also in CIFAR-10, it reveals a failure mode. While AD accurately identifies a strong deviation from normality, the flow matching model struggles to map constant background to typical Gaussian noise even for in-distribution samples. This limitation is likely related to the used model architecture---in this case a UNet. It is worth emphasising that this behaviour reveals an edge case while the method performs well overall as evidenced by the strong AUROC results in \cref{tab:cross_dataset_results} and the visualisation of a larger set of samples in \cref{appendix:top_150_ood}, which shows the 150 images with highest OOD scores.

        \begin{figure}[t]
            \centering
            \includegraphics[width=\textwidth]{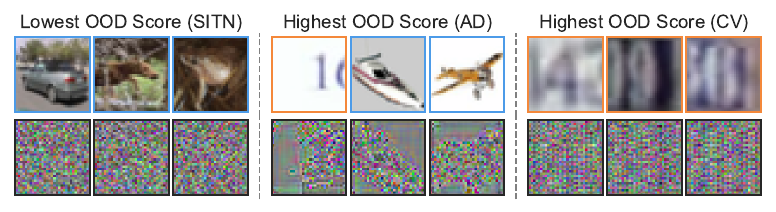}
            \caption{
                Visualisation of images along with their corresponding noise samples. Images from the CIFAR-10$\rightarrow$SVHN setup are shown with ID CIFAR-10 images marked by a \textcolor[RGB]{76, 155, 232}{blue} border and OOD SVHN images by an \textcolor[RGB]{244, 136, 58}{orange} border. \textbf{Left:} samples with lowest SITN OOD score, where the noise is in line with standard Gaussian noise in terms of both the Anderson Darling (AD) statistic and the coefficient of variation (CV) of the power spectrum. \textbf{Middle:} samples with highest OOD score driven by the AD statistic. \textbf{Right:} samples with highest OOD score driven by the CV statistic.
            }
            \label{fig:noise_samples_svhn}
        \end{figure}

\section{Discussion and Conclusions}
\label{sec:discussion}

    In this work, we introduced a framework that casts out-of-distribution detection as a goodness-of-fit test in the noise space of continuous normalising flows. As a concrete instantiation of this framework, we proposed Signal in the Noise (SITN), which probes noise samples against a standard Gaussian prior. SITN exploits the complete factorisation of this prior to operate on individual high-dimensional samples and test them for deviations in marginal normality and dimensional independence. Crucially, SITN offers significant practical and statistical advantages: it requires no access to OOD data, incurs minimal computational overhead---requiring neither additional training nor likelihood evaluations---and allows for the formal derivation of $p$-values for strict Type I error control.

    We rigorously validated the effectiveness of SITN across standard cross-dataset benchmarks and synthetic perturbations, where it consistently outperformed baselines. Moreover, detailed visualisations and empirical analyses revealed that the likelihood-based baselines Typicality and DoSE retain a strong complexity bias inherent to the raw likelihood. SITN, by contrast, does not exhibit this bias.

    Our findings also highlight important avenues for future research. While SITN yields highly competitive overall performance, we observed a specific edge case where the UNet architecture fails to map large, constant background patches to Gaussian noise, causing the Anderson-Darling component to flag the retained spatial structures as atypical. Future work will naturally focus on mitigating this artefact. More broadly, we see great promise in exploring statistics other than AD and CV---such as those that identify other forms of dependence---and in studying and optimising the interplay between the OOD detection method and the underlying generative model. Together, these directions offer a clear path to refine and better understand the design space for OOD detection through goodness-of-fit testing in the noise space.

    Ultimately, by equipping machine learning systems with OOD detection safeguards, methods like SITN bring us a critical step closer to their safe deployment in high-risk, real-world applications.

\begin{ack}
    This project was supported by the Royal Academy of Engineering under the Research Fellowship programme. This work was funded by NatWest Group via the Centre for Purpose-Driven Innovation in Banking.
\end{ack}

\newpage
\sloppy
\printbibliography

\newpage
\appendix
\crefalias{section}{appendix}

\section{Metric Ablations}
\label{appendix:perturbation_ablations}

    \begin{figure}[H]
        \centering
        \includegraphics[width=\textwidth]{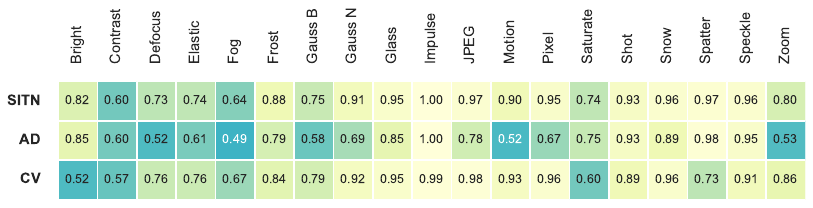}
        \caption{
            Ablation showing the OOD detection performance of the individual statistics \textbf{Anderson Darling (AD)} and the \textbf{coefficient of variation of the power spectrum (CV)} against the \textbf{full method (SITN)}, which combines them. OOD detection performance is shown in terms of AUROC values across the synthetic perturbations from CIFAR-10-C. For each corruption, samples with all severity levels were pooled. The two statistics show complementary value where for example AD is more sensitive to brightness perturbations, whereas CV shows better discrimination for pixellation or motion artefacts.
        }
        \label{fig:perturbations_ablation}
    \end{figure}

\section{Complexity Bias}
\label{appendix:complexity_bias}

    \begin{figure}[H]
        \centering
        \includegraphics[width=\textwidth]{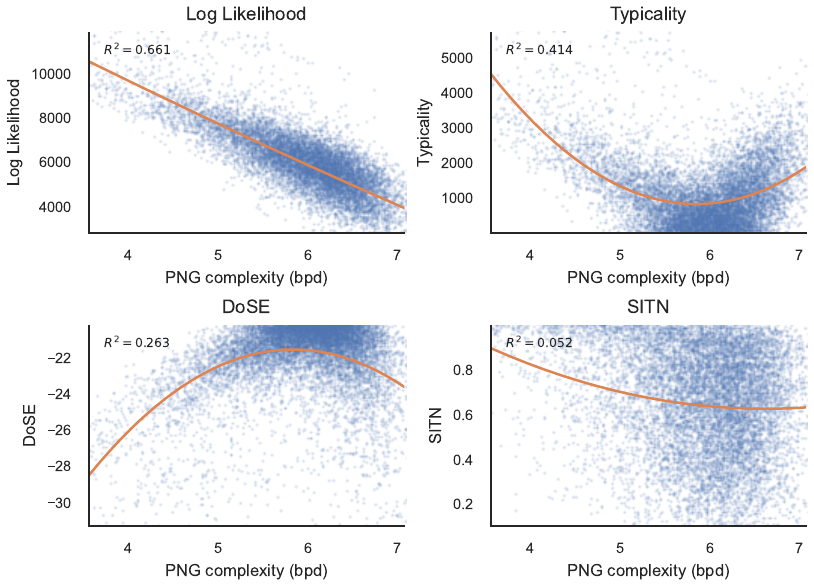}
        \caption{
            Complexity bias of OOD detection metrics. Scatter plots showing the relationship between image complexity (measured by PNG compression size) and the different OOD detection scores for a model trained and evaluated on the CIFAR-10 train and test splits, respectively. The lines show quadratic fits and $R^2$ values are reported alongside. Except for SITN, all metrics show a strong complexity bias.
        }
        \label{fig:complexity_bias}
    \end{figure}

\section{Likelihood Distributions}
\label{appendix:likelihood_distributions}
    
    \begin{figure}[H]
        \centering
        \includegraphics[width=\textwidth]{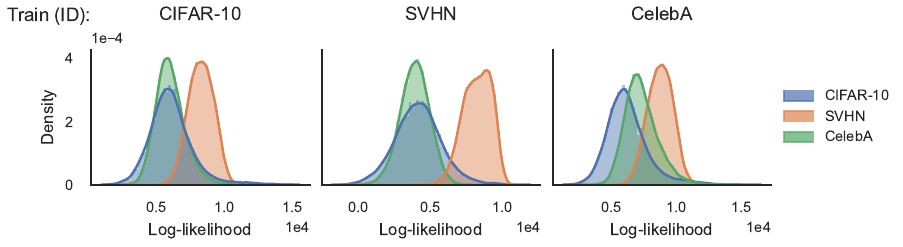}
        \caption{
            Log-Likelihood distributions across the different datasets. Plot titles indicate what dataset the corresponding model was trained on, e.g., the left plot shows the log-likelihood distributions for the model trained on CIFAR-10. In line with prior work, OOD samples from SVHN are consistently assigned higher likelihoods than ID CIFAR-10 (test) samples for the CIFAR-10 model. 
        }
        \label{fig:likelihood_distributions}
    \end{figure}

\section{Metric Extremes on CelebA}
\label{appendix:metric_extremes_celeba}

    \begin{figure}[H]
        \centering
        \includegraphics[width=\textwidth]{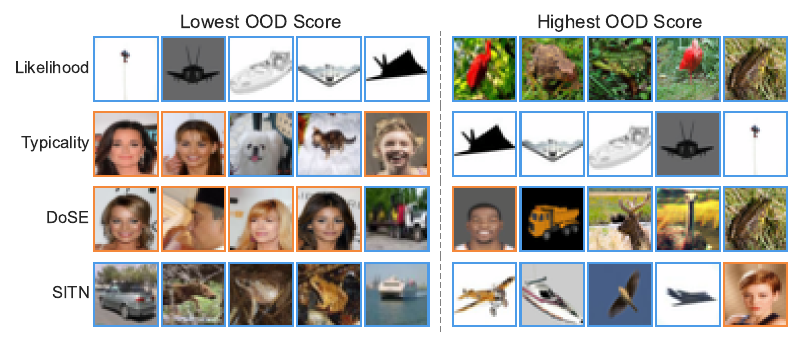}
        \caption{
            Images with the highest and lowest OOD scores for each method in the CIFAR-10$\rightarrow$CelebA setup, i.e. trained on CIFAR-10 and evaluated on the test split of CIFAR-10 as ID and the test split of CelebA as OOD. ID images from CIFAR-10 are marked with a \textcolor[RGB]{76, 155, 232}{blue} border and OOD images from CelebA with an \textcolor[RGB]{244, 136, 58}{orange} border. Both Typicality and DoSE inherit failure cases from the complexity-biased likelihood scores and assign very low OOD scores to OOD images from CelebA.
        }
        \label{fig:metric_extremes_celeba}
    \end{figure}

\section{Additional Baselines}
\label{appendix:additional_baselines}
    In addition to the baseline methods described in \cref{sec:baselines}, we also evaluated the complexity correction method by \citeauthor*{serra_input_2020} \cite{serra_input_2020} and the ensemble method WAIC by \citeauthor*{choi_waic_2019} \cite{choi_waic_2019}.
    
    For the complexity correction, we used PNG compression as the complexity metric. Specifically, PNG complexity $c(\X)$ was computed in units of bits per dimension (bpd) for each image (after resizing to 32x32 for CelebA and without channel-wise normalisation by the dataset mean and standard deviation). As in the original publication, the log-likelihoods $L(\X)$ were then also expressed in bpd and the complexity-corrected OOD score was computed as $S(\X) = -L(\X) - c(\X)$.
    
    We see this method as a less relevant benchmark for SITN due to its additional requirements of a suitable complexity metric that may not be available in less standard domains. Nevertheless, SITN consistently outperforms the complexity correction approach across all conditions (see \cref{tab:additional_baselines}).
    Our evaluations also include results for models trained on SVHN and CelebA, which were not previously reported by \citeauthor*{serra_input_2020} and reveal severe failures of the method for the CelebA$\rightarrow$CIFAR-10 and CelebA$\rightarrow$SVHN setups.

    \begin{table}[h]
      \caption{Comparison to additional baselines for cross-dataset OOD detection performance in terms of AUROC. Values in brackets denote bootstrapped 95\% CIs and results for the best method are highlighted in \textbf{bold}.}
      \label{tab:additional_baselines}
      \centering
      \begin{tabular}{llcccc}
        \toprule
        \textbf{Train (ID)} & \textbf{Test (OOD)} & \textbf{Log-likelihood} & \textbf{Complexity} & \textbf{WAIC} & \textbf{SITN (Ours)} \\
        \midrule
        CIFAR-10 & SVHN     & 0.105\cibracket{.100}{.110} & 0.782\cibracket{.777}{.787} & 0.332\cibracket{.326}{.338} & \textbf{0.946}\cibracket{.943}{.949} \\
                 & CelebA   & 0.496\cibracket{.488}{.503} & 0.615\cibracket{.608}{.623} & 0.450\cibracket{.444}{.458} & \textbf{0.671}\cibracket{.665}{.678} \\
        \midrule
        SVHN     & CIFAR-10 & 0.980\cibracket{.978}{.981} & 0.729\cibracket{.723}{.734} & 0.745\cibracket{.739}{.750} & \textbf{0.986}\cibracket{.984}{.987} \\
                 & CelebA   & \textbf{0.999}\cibracket{.999}{.999} & 0.884\cibracket{.881}{.887} & 0.736\cibracket{.732}{.741} & 0.996\cibracket{.995}{.996} \\
        \midrule
        CelebA   & CIFAR-10 & 0.728\cibracket{.721}{.734} & 0.224\cibracket{.218}{.230} & 0.609\cibracket{.603}{.616} & \textbf{0.910}\cibracket{.907}{.914} \\
                 & SVHN     & 0.179\cibracket{.175}{.183} & 0.077\cibracket{.075}{.080} & 0.449\cibracket{.444}{.454} & \textbf{0.986}\cibracket{.985}{.987} \\
        \bottomrule
      \end{tabular}
    \end{table}
    
    For WAIC, we followed the original publication by training an ensemble of five models with different parameter initialisations. The WAIC score is then computed as the mean log-likelihood minus the variance of the log-likelihoods across the five models, with lower scores interpreted as more out-of-distribution. As explained in \cite{choi_waic_2019}, the variance term is intended to act as a correction term to penalise high epistemic uncertainty in the ensemble.

    Compared to SITN, WAIC introduces substantial computational overhead, requiring the training of an entire ensemble of generative models along with likelihood evaluations for every model in the ensemble during inference time. As shown in \cref{tab:additional_baselines}, SITN nevertheless outperforms WAIC across all conditions. In fact, we find that WAIC performs quite poorly throughout our experiments, which we investigated further in \cref{appendix:waic_failure}.

\section{Further Investigation of WAIC Performance}
\label{appendix:waic_failure}
    In \cref{appendix:additional_baselines}, we reported the performance of WAIC \cite{choi_waic_2019} and observed it to perform quite poorly overall. To further investigate these results, we visualised the distributions of the ensemble mean and variance of the likelihood evaluations. The variance term is intended to serve as a correction term to penalise epistemic uncertainty, which the authors expect to be larger for OOD samples, therefore providing an OOD signal. As shown in \cref{fig:waic_histograms} (right panel), we did not observe this behaviour. The distribution of the ensemble variance was highly overlapping between in- and out-of-distribution samples, explaining why WAIC underperformed in our experiments.

    \begin{figure}[H]
        \centering
        \includegraphics[width=\textwidth]{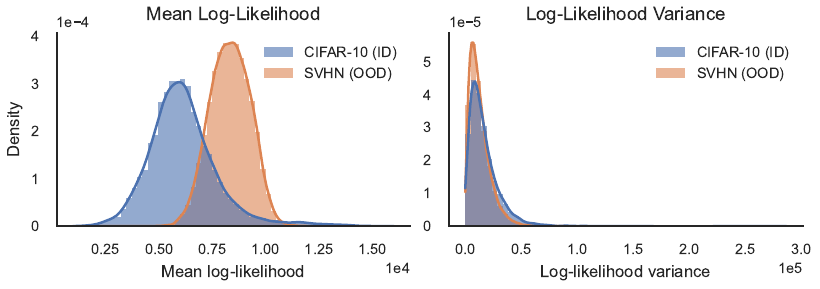}
        \caption{
            Distribution of the ensemble mean and variance of the log-likelihoods across the CIFAR-10 and SVHN test splits for a CIFAR-10-trained ensemble. WAIC is based on the idea of using the variance of the likelihoods (across ensemble models) as an OOD signal. However, in our experiments with flow matching models, we do not observe consistently higher variance for OOD samples.
        }
        \label{fig:waic_histograms}
    \end{figure}

    We hypothesise that the flow matching models used in our setup produce much more stable likelihood evaluations than the ensemble of Glow models used in the original work. This would be in line with recent findings about the stability of flow matching \cite{briq_amazing_2026}.

\section{DoSE and Typicality Fit on Validation Data}
\label{appendix:baselines_on_val}

    \begin{table}[H]
      \caption{
        Control experiments using validation data to fit DoSE and the entropy estimate for Typicality. While both methods use training data in the original publications, these experiments test if instead using validation data meaningfully changes their performance. No consistent improvement is observed: both methods perform slightly better with validation data on SVHN, but worse on CelebA. Rows corresponding to the original configurations used in the main paper are highlighted in gray.
      }
      \label{tab:baselines_on_val}
      \centering
      \begin{tabular}{lllcc}
        \toprule
        \textbf{Train (ID)} & \textbf{Test (OOD)} & \textbf{Fit Split} & \textbf{Typicality} & \textbf{DoSE} \\
        \midrule
        CIFAR-10 & SVHN   & Val   & 0.803\cibracket{.798}{.809} & 0.812\cibracket{.806}{.817} \\
        \rowcolor{gray!15} \cellcolor{white} & \cellcolor{white} & Train & 0.790\cibracket{.785}{.796} & 0.802\cibracket{.796}{.808} \\
        \cmidrule{2-5}
                 & CelebA & Val   & 0.420\cibracket{.413}{.427} & 0.429\cibracket{.422}{.436} \\
        \rowcolor{gray!15} \cellcolor{white} & \cellcolor{white} & Train & 0.421\cibracket{.414}{.428} & 0.432\cibracket{.425}{.439} \\
        \bottomrule
      \end{tabular}
    \end{table}

\section{Calibration}
\label{appendix:calibration}

    In the following, we demonstrate how $S_{\text{SITN}}$ can be calibrated, allowing one to specify a tolerable false positive rate, $\alpha$. This is accomplished by splitting the calibration set, $C$, into disjoint subsets, $C_1$ and $C_2$, then using $C_1$ to estimate the empirical CDFs of the constituent $S_{\text{AD}}$ and $S_{\text{CV}}$ statistics, and $C_2$ to estimate the empirical CDF, $\hat{F}_{S_{\text{SITN}}}$ of $S_{\text{SITN}}$. Finally, one defines the ID versus OOD classification rule as
    \begin{equation}
        h(\X) = \begin{cases*}
            \text{OOD} & if $\hat{F}_{S_{\text{SITN}}}(S_{\text{SITN}}) \geq 1 - \alpha$ \\
            \text{ID} & otherwise.
        \end{cases*}
    \end{equation}
    We show below how this controls false positive (i.e., Type I error) rate, and provide some analysis about false negative (i.e., Type II error) rates under some assumptions.

    \subsection{Type I Error Control}
        A hypothesis test has a Type I error (i.e., false positive) rate of $\alpha$ if the probability of falsely rejecting the null hypothesis, $H_{0}$, is $\alpha$. In the context of OOD detection, this can be stated formally as
        \begin{equation}
            \Pid(h(X) = \textup{OOD}) = \alpha.
        \end{equation}
        We prove Type I error control for a slightly more general class of composite hypotheses that can be decomposed into a conjunction over $k$ constituent hypotheses. The approach generalises by simply estimating the empirical CDF for all $k$ associated statistics, and taking a $k$-way max instead of a two-way max.
        
        Let us write $M = \max_{i} F_{S_i}(S_{i})$;
        then
        \begin{align}
            \Pid(h(X) = \text{OOD}) &= \Pid(F_M(M) \geq 1 - \alpha) \\
            &= 1 - \Pid(F_M(M) < 1 - \alpha) \\
            &= \alpha,
        \end{align}
        because $F_M(M)$ follows a uniform distribution. This tells us that controlling the Type I error rate for the composite hypothesis test can be accomplished by estimating $F_{M}$ with bounded error. We demonstrate that our estimation procedure has bounded error in the theorem below.

        \begin{theorem}
            Let $\hat{M}=\max_{i}\hat{F}_{S_i}(S_{i})$ be our estimate of $M$, where $\hat{F}_{S_i}$ are estimated using $C_1$ and $\hat{F}_{\hat{M}}$ is estimated using $C_2$. Then, with probability at least $1-\delta$,
            \begin{equation*}
                \|F_{M} - \hat{F}_{\hat{M}}\|_\infty \leq 2k\sqrt{\frac{1}{2|C_1|}\log{\frac{2k}{\delta}}} + \sqrt{\frac{1}{2|C_2|}\log{\frac{2}{\delta}}}.
            \end{equation*}
            \label{theorem:empirical_cdf_bound}
        \end{theorem}
        The proof is given in \cref{sec:proof}. Taking $\hat{M} = S_{\text{SITN}}$ and $k=2$ with the two statistics identified in the main paper, we obtain the following corollary for $|C|=n$ and $|C_1|=|C_2|=\frac{n}{2}$.
        \begin{corollary}
            The Type I error of the SITN test is bounded from above, with probability at least $1-\delta$ over the draws of the calibration sets, by
            \begin{equation*}
                \Pid(h(X) = \text{OOD}) \leq \alpha + 4 \sqrt{\frac{1}{n} \log \frac{4}{\delta}} + \sqrt{\frac{1}{n} \log \frac{2}{\delta}}.
            \end{equation*}
        \end{corollary}
        This tells us that, as the size of the calibration sets increases, the control on the Type I error becomes more precise.
    
    \subsection{Proof of \cref{theorem:empirical_cdf_bound}}
    \label{sec:proof}
        The proof of this theorem will make use of the DKW inequality, given below.
        \begin{theorem}[Dvoretzky-Kiefer-Wolfowitz (DKW) inequality \cite{dvoretzky_1956}]
            \label{theorem:dkw}
            Let $\hat{F}$ be an empirical distribution function based on $n$ i.i.d. samples from a distribution function $F$. Then the probability of maximum error between the distributions, $\sup_{x\in \mathbb{R}}(\hat{F}(x) - F(x))$ exceeding $\epsilon \in \mathbb{R}$ can be bounded by
            \begin{equation*}
                P(\sup_{x\in \mathbb{R}}|\hat{F}(x) - F(x)| \geq \epsilon) \leq 2e^{-2n\epsilon^{2}}.
            \end{equation*}
        \end{theorem}
   
        We will also make use of the following Lemma.
        \begin{lemma}
            Let $\hat{F}$ be an empirical distribution function based on $n$ i.i.d. samples from a continuous distribution with CDF $F$. Then, with probability at least $1-\delta$,
            \begin{equation*}
                \|\hat{F}(x)-F(x)\|_\infty \leq \sqrt{\frac{1}{2n} \log{\frac{2}{\delta}}}.
            \end{equation*}
            \label{lemma:dkw_bound}
        \end{lemma}
        
        \begin{proof}[Proof of Lemma \ref{lemma:dkw_bound}]
            Let $\delta = P(\sup_{x\in \mathbb{R}}|\hat{F}(x) - F(x)| \geq \epsilon)$. Then, by the DKW Theorem (Theorem \ref{theorem:dkw})
            \begin{align*}
                \delta \leq 2e^{-2n\epsilon^{2}}\\
                \implies \epsilon \leq \sqrt{\frac{1}{2n} \log{\frac{2}{\delta}}}
            \end{align*}
            Therefore, with at least probability $1-\delta, \sup_{x\in \mathbb{R}}|\hat{F}(x)-F(x)|\leq\epsilon$. It follows that, with probability $1-\delta$, that $\|\hat{F}(x)-F(x)\|_\infty \leq \sqrt{\frac{1}{2n} \log{\frac{2}{\delta}}}$  
        \end{proof}
    
        This lemma can be easily extended to the case with $k$ hypotheses as follows
        \begin{lemma}
            \label{lemma:dkw_K_bound}
            Let $\{F_{i}\}_{i=1}^k$ be a family of continuous distributions and $\{\hat{F}_{i}\}_{i=1}^k$ be a family of empirical distribution functions, where $\hat{F}_{i}$ is estimated based on $n$ i.i.d. draws from the corresponding distribution. Then, with probability $1-\delta$;
            \begin{equation*}
                \max_{i}\|F_{i}(x) - \hat{F}_{i}(x)\|_\infty \leq \sqrt{\frac{1}{2n}\log{\frac{2k}{\delta}}}.
            \end{equation*}
        \end{lemma}
    
        \begin{proof}[Proof of Lemma \ref{lemma:dkw_K_bound}]
            This result follows from a straightforward application of the union bound to the DKW in Lemma \ref{lemma:dkw_bound}. With probability at least $1-\delta$
            \begin{equation*}
                P(\exists i \in \{1...k\}: \max_{x}|F_{i}(x)-\hat{F}_{i}(x)|\geq\epsilon) \leq \sum_{i=1}^k\mathbb{P}(|\max_{x}F_{i}(x) - \hat{F}_{i}(x)|\geq\epsilon) \leq 2ke^{-2n\epsilon^{2}}
            \end{equation*}
            If $\delta = P(\exists i \in \{1...k\}: \sup_{x}|F_{i}(x)-\hat{F}_{i}(x)|\geq\epsilon)$ then, for $\epsilon = \sqrt{\frac{1}{2n}\log{\frac{2k}{\delta}}}$, with probability $1-\delta$, we have that
            \begin{equation*}
                \max_i\sup_{x}|F_{i}(x) - \hat{F}_{i}(x)| \leq \epsilon.
            \end{equation*}
        \end{proof}
    
        The proof of Theorem \ref{theorem:empirical_cdf_bound} will also make use of the following Lemma.
        \begin{lemma}
            \label{lemma:cdf_M_vs_ecdf_M_bound}
            Let $F_{i}$ be the CDF over test statistic $S_{i}$, and $\hat{F}_{i}$ the empirical CDF of $F$ based on $n$ i.i.d. samples. If $M = \max_{i} F_{i}(S_{i})$ and $\hat{M} = \max_{i} \hat{F}_{i}(S_{i})$ then, with probability at least $1-\delta$,
            \begin{equation*}
                \|F_{M} - F_{\hat{M}}\|_\infty \leq 2k\sqrt{\frac{1}{2n}\log{\frac{2k}{\delta}}}.
            \end{equation*}
        \end{lemma}
    
        \begin{proof}[Proof of Lemma \ref{lemma:cdf_M_vs_ecdf_M_bound}]
            Construct an upper bound on $|M - \hat{M}|$;
            \begin{align*}
                |M-\hat{M}| &= |\max_{i}F_{i}(S_{i}) - \max_{i}\hat{F}_{i}(S_{i})|\\
                &\leq \max_{i}|F_{i}(S_{i}) - \hat{F}_{i}(S_{i})| \leq \epsilon
            \end{align*}
            with probability $1-\delta$ for $\epsilon = \sqrt{\frac{1}{2n}\log{\frac{2k}{\delta}}}$. This result follows from the monotonicity of the max operation and application of Lemma \ref{lemma:dkw_K_bound}.
    
            Using this bound observe that, if  $M\leq t-\epsilon$ then $\hat{M} \leq t$. Similarly, $\hat{M} \leq t \implies M \leq t+\epsilon$. Therefore $F_{M}(t-\epsilon) \leq F_{\hat{M}}(t) \leq F_{M}(t+\epsilon)$. It follows that
            \begin{alignat*}{2}
                &|F_{\hat{M}}(t)-F_{M}(t)| &&\leq \max(F_{M}(t+\epsilon) - F_{M}(t), F_{M}(t)-F_{M}(t-\epsilon)) \\
                & &&\leq F_{M}(t+\epsilon) -F_{M}(t-\epsilon)\\
                \implies \sup_{t} &| F_{\hat{M}}(t) - F_{M}(t)| &&\leq \sup_{t}(F_{M}(t+\epsilon) -F_{M}(t-\epsilon)).
            \end{alignat*}
    
            Finally, observe that
            \begin{align*}
                F_{M}(t+\epsilon) - F_{M}(t-\epsilon) &= P(\max_{i}S_{i} \leq t+\epsilon) - P(\max_{i}S_{i}\leq t-\epsilon)\\
                & \leq \sum_{i=1}^k P(U_{i} \in (t-\epsilon, t+\epsilon]) 
            \end{align*}
            by application of the union bound. Since $U_{i} = F_{i}(S_{i})$ it is uniformly distributed in $[0,1]$, we have that $P(U_{i} \in (t-\epsilon, t+\epsilon]) = 2\epsilon$. And so, with probability at least $1-\delta$
            \begin{align*}
                \sup_{t}|F_{M}(t) - F_{\hat{M}}(t)| &\leq 2k\epsilon\\
                &= 2k\sqrt{\frac{1}{2n}\log{\frac{2k}{\delta}}}
            \end{align*}
            completing the proof.
        \end{proof}
    
        \begin{proof}[Proof of Theorem \ref{theorem:empirical_cdf_bound}]
            Rewrite $\|F_{M} - \hat{F}_{\hat{M}}\|_\infty$ using the triangle inequality, as follows
            \begin{align*}
                \|F_{M} - \hat{F}_{\hat{M}}\|_\infty &= \|F_{M} - F_{\hat{M}} + F_{\hat{M}} - \hat{F}_{\hat{M}}\|_\infty \\
                &\leq \|F_{M} - F_{\hat{M}}\|_\infty + \|F_{\hat{M}}- \hat{F}_{\hat{M}}\|_\infty.
            \end{align*}
            By application of Lemma \ref{lemma:cdf_M_vs_ecdf_M_bound} and Lemma \ref{lemma:dkw_bound}
            \begin{equation*}
                \|F_{M} - F_{\hat{M}}\|_\infty + \|F_{\hat{M}}- \hat{F}_{\hat{M}}\|_\infty \leq 2k\sqrt{\frac{1}{2|C_{1}|}\log{\frac{2k}{\delta}}} + \sqrt{\frac{1}{2|C_{2}|}\log{\frac{2}{\delta}}}
            \end{equation*}
        \end{proof}

    \subsection{Empirical Demonstration of Type I Error Control}

        \cref{fig:calibration_curves} (left panel) shows calibration curves for the CIFAR-10-trained model evaluated on the CIFAR-10 test data. As expected, the raw max-quantile SITN scores are not calibrated (blue line). While the individual quantile-transformed statistics $\hat{F}_{\text{AD}}(S_{\text{AD}})$ and $\hat{F}_{\text{CV}}(S_{\text{CV}})$ could be interpreted as the fraction of validation samples with equal or lower scores, this interpretation no longer holds after their max-aggregation into the combined SITN OOD score. Using the method described above (\cref{appendix:calibration}), we obtain a well-calibrated score (orange line). We also show Tippett's method (green line) for comparison, which assumes independence between the combined components, and observe that it produced only slightly more conservative scores in this case.

        For comparison and illustration purposes, we also show calibration curves for the same calibration approaches for the combination of two statistics where the independence assumption for Tippett's method is clearly violated (\cref{fig:calibration_curves}, right panel). In this example, we combine the Anderson-Darling statistic and the Kolmogorov-Smirnov statistic, both of which are sensitive to deviations from normality. It can be seen that Tippett's method is much too conservative in this case, whereas calibration on the split validation data, which does not assume independence of the statistics, still yields well-calibrated results. Consequently, when exploring new combinations of test statistics in future work, we caution against relying on blanket independence assumptions and recommend using the procedure described above.
    
        \begin{figure}[H]
            \centering
            \includegraphics[width=2.5in]{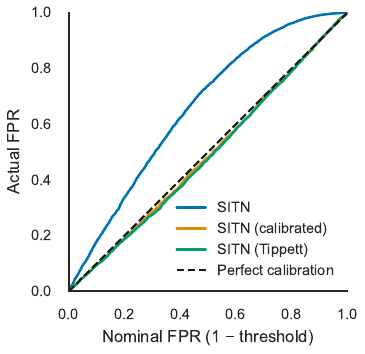}
            \hfill
            \includegraphics[width=2.5in]{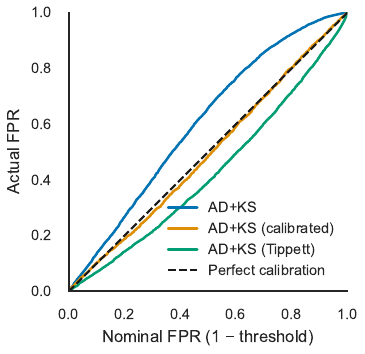}
            \caption{
                Calibration curves showing the actual false positive rate (FPR)---the fraction of in-distribution samples flagged as OOD---as a function of the nominal FPR set by the detection threshold on the CIFAR-10 test data. A perfectly calibrated method lies on the diagonal (dashed). \textbf{Left:} It can be seen that the raw max quantile scores (\textcolor[RGB]{1, 115, 178}{SITN}) are uncalibrated, whereas Tippett's method (\textcolor[RGB]{2, 158, 115}{Tippett}) and calibration on the split validation data (\textcolor[RGB]{222, 143, 5}{calibrated}) provide similar and well-calibrated results. \textbf{Right:} For comparison, calibration with the split validation data (\textcolor[RGB]{222, 143, 5}{calibrated}) remains effective if two statistics that clearly violate the independence assumption for Tippett's method (\textcolor[RGB]{2, 158, 115}{Tippett}) are combined (here Anderson-Darling and Kolmogorov-Smirnov).
            }
            \label{fig:calibration_curves}
        \end{figure}

    \subsection{Type II Error Control}

        \def\pood{p_{\text{OOD}}}
        \def\pid{p_{\text{ID}}}
        \def\pnoise{p_{\text{0}}}
        \def\palt{p_{\text{A}}}
        
        The Type II error rate is given by
        \begin{align}
            \Pood(h_\gamma(X) = \text{ID}) &= \Pood(F_{S_{\text{AD}}}(S_{\text{AD}}) \leq \gamma \land F_{S_{\text{CV}}}(S_{\text{CV}}) \leq \gamma),
        \end{align}
        where we use $\gamma$ to denote $F_{S_{\text{SITN}}}^{-1}(1-\alpha)$. We have no access to data from $\Pood$, so it is not possible to control this quantity without making some assumptions about the relationship between $\Pood$ and $\Pid$. Because we will now make assumptions about the form of $\Pood$, we take the liberty of making a few other assumptions that will ease the analysis without substantially changing the conclusions. In particular, we assume that $\phi$ perfectly models $\Pid$, that $\mathbb{E}[X] = 0$ for $X \sim \Pid$, and that the calibration set, $C$, is infinitely large. I.e., our empirical CDFs equal the true CDFs.
        
        We define $\Pood$ in terms of how its density is related to that of $\Pid$. In particular, denoting density functions using lower case $p$, for some positive real $a$ we consider
        \begin{equation}
            \pood(X) = a^{-D} e^{-\frac{1 - a^2}{a^2}\|X\|_2^2}\pid(X).
        \end{equation}
        For $a=1$ this sets $\Pood = \Pid$, but for $a\neq1$ it will result in more density being assigned to the tails or centre of the distribution. We choose this particular re-weighting because it results in a convenient distribution for $\phi^{-1}$ when $X \sim \Pood$, as stated in the lemma below. However, we also feel that it matches the intuition that many people have about OOD data being more likely to be found in the tails of the distribution, and it is a natural way to parameterise the relationship between $\Pood$ and $\Pid$ in terms of a single parameter $a$ that controls how much more likely OOD data is to be found in the tails compared to ID data.
        
        \begin{lemma}
            Under the above definition of $\Pood$, we have that
            \begin{equation}
                \textup{Law}(\phi^{-1}(X)) = \mathcal{N}(0, a^2I_{d}),\qquad \X \sim \Pood.
            \end{equation}
        \end{lemma}
        \begin{proof}
            We can use the change of variable formula to derive the density of $\phi^{-1}(X)$ in the situations where $\X \sim \Pid$ and $\X \sim \Pood$. In the first case we have
            \begin{equation}
                \pnoise(\phi^{-1}(\X)) = \pid(\X)|\text{det}J_{\phi^{-1}}(\X)|,
            \end{equation}
            and in the second we have that
            \begin{equation}
                \palt(\phi^{-1}(\X)) = \pood(\X)|\text{det}J_{\phi^{-1}}(\X)|.
            \end{equation}
            Taking the ratio of these, we obtain
            \begin{align*}
            \palt(\phi^{-1}(\X)) &= \frac{\pood(\X)}{\pid(\X)}\pnoise(\phi^{-1}(\X))\\
                                 &= a^{-D} e^{-\frac{1-a^2}{a^2}\|\X\|_2^2} \mathcal{N}(x; 0, I)\\
                &= \mathcal{N}(x; 0, a^2I).
            \end{align*}
        \end{proof}
        
        We can decompose the Type II error as
        \begin{align}
            &\Pood(F_{S_{\text{CV}}}(S_{\text{CV}}) \leq \gamma | F_{S_{\text{AD}}}(S_{\text{AD}}) \leq \gamma) \Pood(F_{S_{\text{AD}}}(S_{\text{AD}}) \leq \gamma) \\
            &=\Pid(F_{S_{\text{CV}}}(S_{\text{CV}}) \leq \gamma | F_{S_{\text{AD}}}(S_{\text{AD}}) \leq \gamma) \Pood(F_{S_{\text{AD}}}(S_{\text{AD}}) \leq \gamma),
        \end{align}
        because the change in variance of the noise distribution for OOD compared to ID data does not result in a change in the CV statistic, due to the scale invariance of this statistic. Now we note that
        \begin{align}
            \Pid(F_{S_{\text{CV}}}(S_{\text{CV}}) \leq \gamma | F_{S_{\text{AD}}}(S_{\text{AD}}) \leq \gamma) &= \frac{1-\alpha}{\Pid(F_{S_{\text{AD}}}(S_{\text{AD}}) \leq \gamma)} \\
                    &=\frac{1-\alpha}{\gamma},
        \end{align}
        due to the marginal uniform distribution induced by the quantile transform. This yields a Type II error rate of
        \begin{align}
            \Pood(h_\gamma(X) = \text{ID}) &= \frac{1-\alpha}{\gamma} \Pood(F_{S_{\text{AD}}}(S_{\text{AD}}) \leq \gamma) \\
            &= \frac{1 - \alpha}{F_{S_{\text{SITN}}}^{-1}(1-\alpha)} \Pood(F_{S_{\text{AD}}}(S_{\text{AD}}) \leq \gamma).
        \end{align}
        From this, we can conclude that the Type II error rate of SITN for the defined $\Pood$ is dependent mainly on the Type II error of the Anderson-Darling test, premultiplied by a factor that can be interpreted as measuring the calibration of the $S_{\text{SITN}}$ test statistic. I.e., if $S_{\text{SITN}}$ is already perfectly calibrated it will follow a uniform distribution, and the premultiplication factor will be one. Prior work has shown that Anderson-Darling tests exhibit favourable Type II error rates in the setting where the alternative hypothesis are zero mean, non-unit variance normal distributions, such as the case of our conjectured $\Pood$.

\section{Training Details}
\label{appendix:training_details}

    We used the UNet architecture following \citeauthor*{ho_denoising_2020} \cite{ho_denoising_2020}, as implemented in \citeauthor*{dhariwal_diffusion_2021} \cite{dhariwal_diffusion_2021} with the hyperparameters detailed in \cref{tab:hyperparameters}.
    Models were optimized with AdamW (learning rate 1e-4, weight decay $0.01$, $\beta_1=0.9$, $\beta_2=0.999$, $\epsilon=1\text{e-}8$) and global gradient norm clipping (max norm $1.0$) for 500,000 gradient steps (batch size 128), with weights restored to the checkpoint of lowest validation loss.

    All images were normalised per-channel using statistics computed from the respective training set. No data augmentation was applied during training.

    \begin{table}[h]
      \caption{UNet hyperparameters}
      \label{tab:hyperparameters}
      \centering
      \begin{tabular}{ll}
        \toprule
        Hyperparameter & Value \\
        \midrule
        Channels & 128 \\
        Depth & 2 \\
        Channels multiple & 1, 2, 2, 2 \\
        Heads & 1 \\
        Attention resolution & 16 \\
        Dropout & 0.0 \\
        
        \bottomrule
      \end{tabular}
    \end{table}

    Models were trained and evaluated on an internal cluster with L40S GPUs. On this hardware, the training time for an individual model was around 20 hours and likelihood evaluations took about 6 minutes per 1{,}000 samples. The total compute time for the cross-dataset results (\cref{sec:cross_dataset_results}) is therefore around 100 hours (training 3 models, evaluating likelihoods on training and validation data for each dataset required for the fitting of Typicality, DoSE and SITN, and evaluating likelihoods on the test splits of all three datasets for each individual model). Additional evaluations for the perturbation experiments (\cref{sec:perturbation_results}) take roughly another 100 hours (50{,}000 samples for each of the 19 corruptions). The additional baseline experiments with the WAIC ensemble approach (\cref{appendix:additional_baselines}) require the training and evaluation of an additional four models (for a total of 5 models in the ensemble) for the aforementioned cross-dataset evaluations, resulting in another 400 hours of compute time.

\section{Statistical Analysis}
\label{appendix:statistical_analysis}

    Confidence intervals for AUROC values were computed through a stratified bootstrap resampling procedure with 10,000 iterations. In each iteration, data was sampled with replacement while preserving the original ratio of in-distribution to out-of-distribution samples. The confidence intervals were derived using the empirical 2.5th and 97.5th percentiles of the resulting metric distributions (i.e., the percentile bootstrap method).
    
\section{Software}
\label{appendix:software}

    The research code for our experiments is publicly available at \reviewhide{\url{https://github.com/bomatter/signal-in-the-noise-paper}}. Dependencies notably include PyTorch \cite{paszke_pytorch_2019}, the flow matching library by \citeauthor*{lipman_flow_2024} \cite{lipman_flow_2024} for implementations of the ODE solvers for the likelihood evaluations, and the UNet implementation by \citeauthor*{dhariwal_diffusion_2021}\cite{dhariwal_diffusion_2021}.

\section{Top 150 Samples with Highest OOD Scores}
\label{appendix:top_150_ood}

    In the following we visualise the top 150 samples with the highest OOD scores for the CIFAR-10$\rightarrow$SVHN setup for each of the baseline methods and SITN. For all baseline methods, the highest OOD scores are predominantly assigned to in-distribution samples (failure cases), whereas the highest SITN OOD scores mostly correspond to actual OOD samples.

    \begin{figure}[h]
        \centering
        \includegraphics[width=\textwidth]{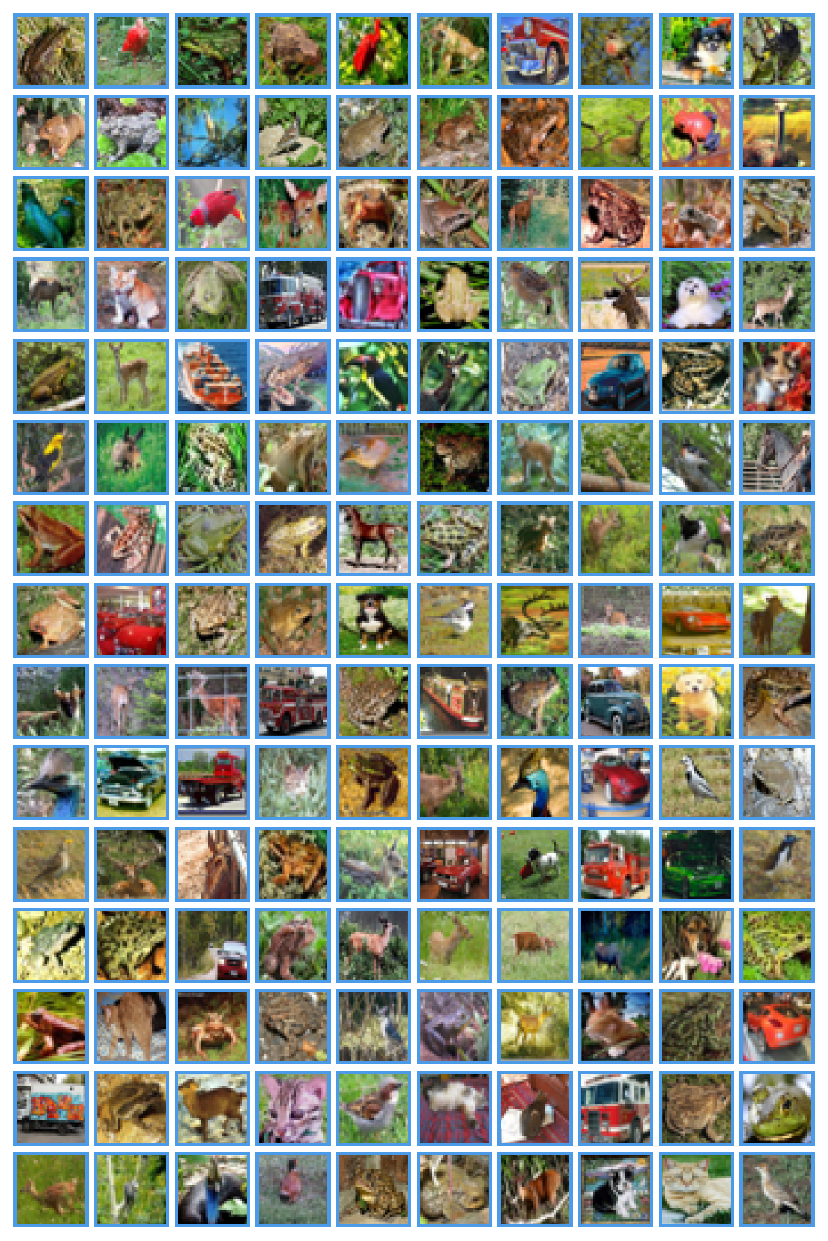}
        \caption{
            Top 150 samples with highest log-likelihood OOD score (lowest log-likelihoods) for the CIFAR-10$\rightarrow$SVHN setup. ID images from CIFAR-10 and OOD images from SVHN are marked with \textcolor[RGB]{76, 155, 232}{blue} and \textcolor[RGB]{244, 136, 58}{orange} borders, respectively.
            All 150 samples flagged as most OOD by the likelihood metric are in-distribution (blue), illustrating a pronounced failure of the method in this setup.
        }
        \label{fig:top_150_cifar10_to_svhn_log_likelihood}
    \end{figure}

    \begin{figure}[h]
        \centering
        \includegraphics[width=\textwidth]{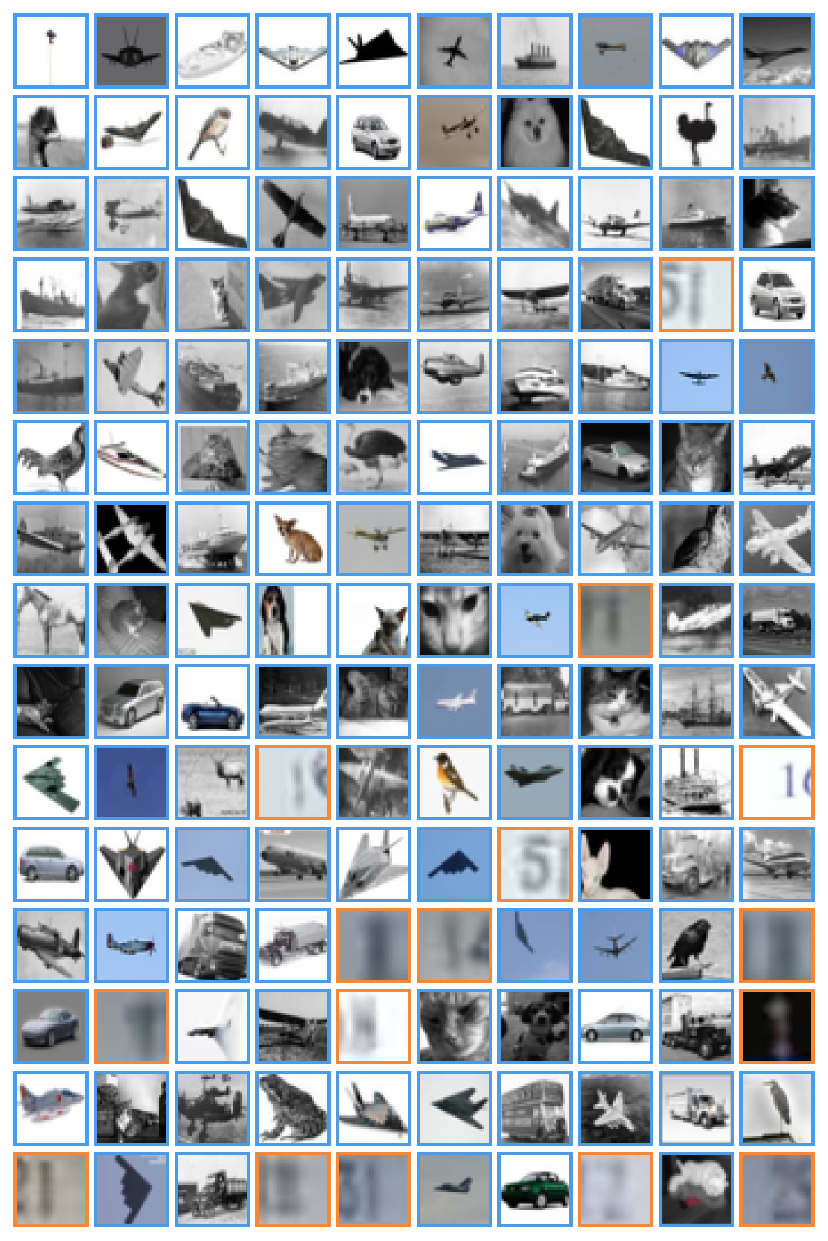}
        \caption{
            Top 150 samples with highest Typicality OOD score for the CIFAR-10$\rightarrow$SVHN setup. ID images from CIFAR-10 and OOD images from SVHN are marked with \textcolor[RGB]{76, 155, 232}{blue} and \textcolor[RGB]{244, 136, 58}{orange} borders, respectively.
            Samples flagged as most OOD by the Typicality metric are dominated by simple in-distribution images, illustrating the method's complexity bias and pronounced failure mode.
        }
        \label{fig:top_150_cifar10_to_svhn_typicality}
    \end{figure}

    \begin{figure}[h]
        \centering
        \includegraphics[width=\textwidth]{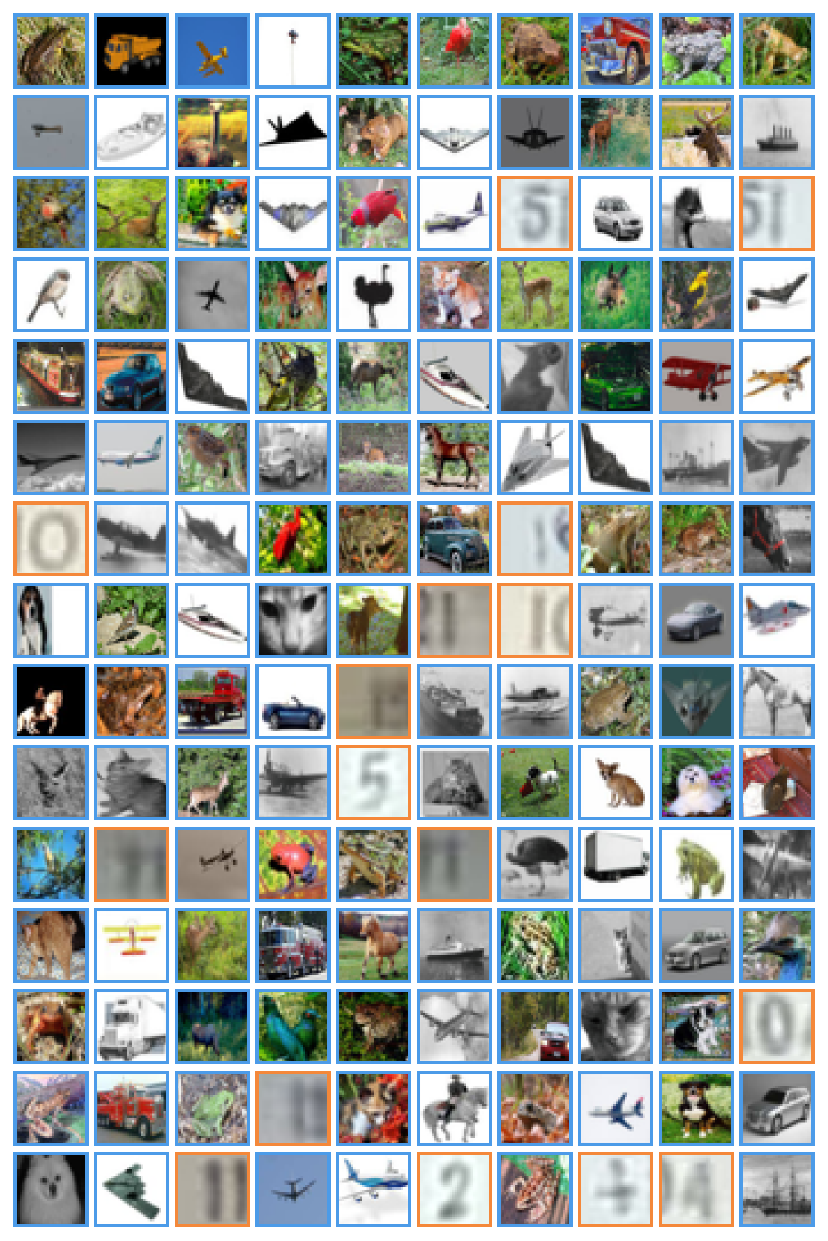}
        \caption{
            Top 150 samples with highest DoSE OOD score for the CIFAR-10$\rightarrow$SVHN setup. ID images from CIFAR-10 and OOD images from SVHN are marked with \textcolor[RGB]{76, 155, 232}{blue} and \textcolor[RGB]{244, 136, 58}{orange} borders, respectively.
            Samples flagged as most OOD by DoSE are dominated by ID images. Shared failure cases with the likelihood and Typicality illustrate a persistent complexity bias and failure mode.
        }
        \label{fig:top_150_cifar10_to_svhn_dose}
    \end{figure}
    
    \begin{figure}[h]
        \centering
        \includegraphics[width=\textwidth]{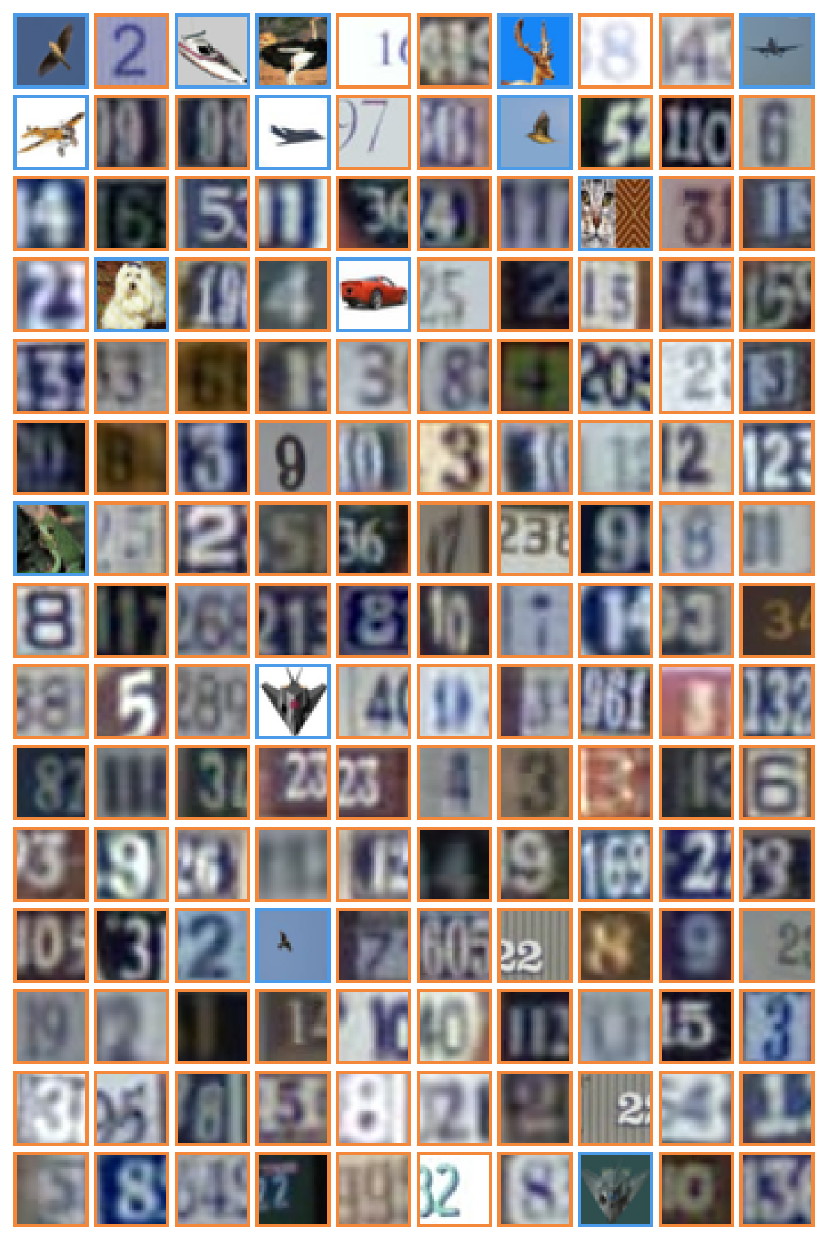}
        \caption{
            Top 150 samples with highest SITN OOD score for the CIFAR-10$\rightarrow$SVHN setup. ID images from CIFAR-10 and OOD images from SVHN are marked with \textcolor[RGB]{76, 155, 232}{blue} and \textcolor[RGB]{244, 136, 58}{orange} borders, respectively. With a few exceptions, the highest SITN OOD scores successfully identify true out-of-distribution samples.
        }
        \label{fig:top_150_cifar10_to_svhn_sitn}
    \end{figure}

\end{document}